\documentclass[letterpaper, 10pt, conference]{ieeeconf}

\usepackage{graphics}
\usepackage{epsfig}

\usepackage{cite}
\usepackage[dvipsnames, svgnames]{xcolor}
\usepackage{xurl}
\usepackage{hyperref}
\hypersetup{
    colorlinks,
    linkcolor={red!50!black},
    citecolor={blue!50!black},
    urlcolor={blue!80!black}
}
\usepackage{bm}
\usepackage{breqn}
\usepackage{amssymb}
\usepackage{tabularx}
\usepackage{booktabs}
\usepackage{dcolumn}
\newcolumntype{d}[1]{D..{#1}}
\newcolumntype{C}{>{\centering\arraybackslash}X}
\usepackage{soul}
\usepackage[linesnumbered,ruled,vlined]{algorithm2e}

\usepackage{amsthm}

\newcommand{\E}{\mathbb{E}}

\newcommand{\Bcal}{\mathcal{B}}
\newcommand{\Ccal}{\mathcal{C}}
\newcommand{\Dcal}{\mathcal{D}}

\newcommand{\Ocal}{\mathcal{O}}
\newcommand{\Pcal}{\mathcal{P}}
\newcommand{\Rcal}{\mathcal{R}}

\newcommand{\Tcal}{\mathcal{T}}

\newcommand{\Vcal}{\mathcal{V}}

\newcommand{\eqn}[1]{\begin{align} #1 \end{align}}

\DeclareMathOperator*{\argmax}{arg\,max}

\theoremstyle{plain}

\newtheorem{lemma}{Lemma}
\theoremstyle{definition}
\newtheorem{definition}{Definition}
\newtheorem*{overallprob*}{Overall Problem}

\theoremstyle{remark}

\definecolor{yellow}{cmyk}{0.0,0.10,0.95,0.0}
\definecolor{pred}{cmyk}{0,0.8,0.70,0.0}
\definecolor{bluedefined}{cmyk}{0.46, 0.10, 0, 0.0}

\usepackage{makecell}
\usepackage[table]{xcolor}

\IEEEoverridecommandlockouts  
\usepackage[english]{babel}
\usepackage{footnotehyper}

\usepackage{hyperref,tablefootnote,footnotehyper}
\usepackage{amsmath,booktabs,mwe,pdflscape}
\hypersetup{colorlinks}

\usepackage{leftindex} 
\usepackage{graphicx} 
\usepackage{amsmath}
\usepackage{amssymb}
\usepackage{amsthm}
\usepackage{amsfonts}
\usepackage{float}
\usepackage{mathrsfs}
\usepackage{booktabs}
\usepackage{array}
\usepackage[nodisplayskipstretch]{setspace}
\makeatletter
\def\BState{\State\hskip-\ALG@thistlm}
\makeatother
\usepackage[capitalize,nameinlink,noabbrev]{cleveref} 
\theoremstyle{plain}

\newtheorem{thm}{Theorem}

\theoremstyle{plain}
\newtheorem{prob}{Problem}
\newtheorem{overprob}{Overall}
\theoremstyle{definition}
\newtheorem{assumption}{Assumption}

\newcolumntype{M}[1]{>{\centering\arraybackslash}m{#1}}
\usepackage{pifont}
\usepackage{multirow}
\usepackage{tabularray}
\usepackage{cite}
\usepackage{algpseudocode}

\newcommand{\RNum}[1]{\uppercase\expandafter{\romannumeral #1\relax}}
\usepackage{bigints}

\algrenewcommand\textproc{}

\usepackage[font=small,skip=10pt]{caption}

\algrenewcommand\algorithmicrequire{\textbf{Input:}}
\algrenewcommand\algorithmicensure{\textbf{Output:}}

\hypersetup{
  colorlinks,
  citecolor=teal, 
  linkcolor=teal,
  urlcolor=purple}

\author{Kaleb Ben Naveed$^{1}$, Haejoon Lee$^{1}$ and Dimitra Panagou$^{1,2}$%
\thanks{$^{*}$The authors would like to acknowledge the support of the National Science Foundation (NSF) under grant no. 2223845 and grant no. 1942907.}
\thanks{$^{1}$Department of Robotics, University of Michigan, Ann Arbor, MI, 48109 USA. 
{\tt\small \{kbnaveed@umich.edu\}}}
\thanks{$^{2}$Department of Aerospace Engineering, University of Michigan, Ann Arbor, MI, 48109 USA. }}

\title{\LARGE \bf
Multi-Robot Allocation for Information Gathering in Non-Uniform Spatiotemporal Environments
} 

\newcommand\footnoteref[1]{\protected@xdef\@thefnmark{\ref{#1}}\@footnotemark}

\begin{document}

\maketitle
\thispagestyle{empty}
\pagestyle{empty}

\begin{abstract}
Autonomous robots are increasingly deployed to estimate spatiotemporal fields (e.g., wind, temperature, gas concentration) that vary across space and time. We consider environments divided into non-overlapping regions with distinct spatial and temporal dynamics, termed non-uniform spatiotemporal environments. Gaussian Processes (GPs) can be used to estimate these fields. The GP model depends on a kernel that encodes how the field co-varies in space and time, with its spatial and temporal lengthscales defining the correlation. Hence, when these lengthscales are incorrect or do not correspond to the actual field, the estimates of uncertainty can be highly inaccurate. Existing GP methods often assume one global lengthscale or update only periodically; some allow spatial variation but ignore temporal changes. To address these limitations, we propose a two-phase framework for multi-robot field estimation. Phase 1 uses a variogram-driven planner to learn region-specific spatial lengthscales. Phase 2 employs an allocation strategy that reassigns robots based on the current uncertainty, and updates sampling as temporal lengthscales are refined. For encoding uncertainty, we utilize clarity, an information metric from our earlier work. We evaluate the proposed method across diverse environments and provide convergence analysis for spatial lengthscale estimation, along with dynamic regret bounds quantifying the gap to the oracle’s allocation sequence.
\end{abstract}



\section{Introduction}
Autonomous robots are increasingly used in environmental monitoring \cite{agrawal2024multi, env_monitoring_2}, search and rescue \cite{moon2022tigris, search_2}, and 3D reconstruction \cite{reconstruction_1, reconstruction_2}, where they must actively gather information to estimate an unknown field. Unlike passive sensing, this requires sequential decision-making, i.e., deciding where and when to sample to reduce uncertainty~\cite{luo2024act}. In \textit{spatiotemporal environments}, where the quantity of interest evolves over space and time, the challenge intensifies. We focus on \textbf{\textit{non-uniform spatiotemporal environments}}, where the domain is partitioned into non-overlapping regions differing in spatial and temporal variability, as shown in~\cref{fig:mot_prop}(a).

Our goal is to estimate the spatiotemporal field by selecting informative sampling locations for collecting measurements and fusing them into a Gaussian Process (GP) representation that predicts the field and quantifies its uncertainty. GPs are widely used in such problems \cite{GPR_informative_1, GPR_informative_3, agrawal2024multi, SpaTemp_2, SpaTemp_3, SpaTemp_4}, although occupancy grids \cite{SpaTemp_Occupancy} and stochastic cell models \cite{Naveed-ICRA-2024} have also been explored. GPs are particularly effective for capturing spatial and temporal correlations via kernels parameterized by lengthscales. However, misspecified lengthscales lead to poor predictions, especially in non-uniform environments. In our setting, each region is modeled with constant spatial and temporal lengthscales, capturing local variability while differing across regions.

Most GP-based methods assume a single global lengthscale \cite{agrawal2024multi, SpaTemp_2, SpaTemp_3, kontoudis2023adaptive} or update GPs periodically \cite{SpaTemp_1}, limiting performance in non-uniform settings. Some allow spatially-varying lengthscales, but ignore temporal variability \cite{Ak-RSS-22}. Others represent the environment using a discretized grid, modeling each grid cell as an independent stochastic process and thereby not leveraging spatial and temporal correlations \cite{Naveed-ICRA-2024}. To our knowledge, GP-based robotic information gathering has not addressed environments with regionally varying spatial and temporal correlations.

To address this gap, we propose a \textbf{\textit{two-phase framework}}. In Phase 1, spatial lengthscales are estimated. We introduce a planner that minimizes the variance of the estimate while ensuring sufficient spatial coverage. This coverage is critical, as variogram-based methods infer correlation by comparing measurements at different distances; without a range of such distances, the estimates can become unreliable. The estimated spatial lengthscales for all regions are then fixed for the remainder of the mission.

In Phase 2, the system operates in closed loop. Robots are allocated to regions with initial sampling locations. At each replanning step, collected data is used to update the region’s temporal lengthscale. These estimates, together with the fixed spatial ones from Phase 1, update the GP model. Robot assignments and sampling locations are then recomputed, adapting at every step to changes in temporal lengthscales, while treating the spatial lengthscale in each region as fixed.

To guide adaptive reallocation, we introduce a clarity-based strategy. Clarity, introduced in our earlier work in~\cite{Agrawal-CDC-23}, is a rescaling of differential entropy so that it takes values on $[0,1]$. This rescaling makes the quantification of the level of uncertainty in the estimate of the field both more intuitive and more amenable to computational derivations. Unlike existing approaches~\cite{Task_Allocation, jang2018anonymous, choudhury2022dynamic}, we formally pose a clarity-maximization problem and reformulate its NP-hard mixed-integer nonlinear program (MINLP) as a tractable mixed-integer linear program (MILP), enabling principled, real-time reallocation.

This framework addresses limitations in prior work. Methods such as ergodic search \cite{Naveed-ICRA-2024, Ergodic_1}, greedy selection \cite{Ak-RSS-22}, and sampling-based planning \cite{moon2022tigris} often yield uneven revisit rates and resolution, degrading estimation in non-uniform environments. Our method learns local structure and adjusts sampling effort accordingly.

\subsubsection*{\textbf{Notation}}
Let $\mathbb{R}$, $\mathbb{R}_{\geq 0}$, and $\mathbb{R}_{> 0}$ denote the set of real, non-negative real, and positive real numbers, respectively. Let $\mathbb{S}^{n}_{+}$ and $\mathbb{S}^{n}_{++}$ denote the sets of symmetric positive semi-definite and positive definite matrices in $\mathbb{R}^{n \times n}$. A Gaussian with mean $\mu$ and covariance $\Sigma \in \mathbb{S}_{++}^{n}$ is denoted by $\mathcal{N}(\mu, \Sigma)$. A L1-norm is defined as $\psi(x) = \frac{x}{\|x\|_1} = \frac{x}{\sum_{i=1}^n |x_i|}$.

\section{Preliminaries}\label{sec:prelim}

\subsection{Clarity}
We use \textit{clarity} \cite{Agrawal-CDC-23}, an information measure that quantifies the level
of information about a continuous random variable $X \in \mathbb{R}^n$ on a normalized $[0,1]$ scale. Let the differential entropy of $X$ be denoted $h[X]$ and defined as
\eqn{
\label{diffentropy}
h[X] &= -\int_S \rho(x)\log \rho(x) dx,
}
where $S$ is the support of $X$. The clarity of $X$ is denoted as $q[X]$ and is defined as
\eqn{
q[X] &= \left(1 + \frac{e^{2h[X]}}{(2\pi e)^n} \right)^{-1},
}
where $h[X]$ is given by \eqref{diffentropy}. Note that $q \rightarrow 1$ represents the case when $X$ is perfectly known, whereas lower values correspond to higher uncertainty. For $Y \sim \mathcal{N}(\mu, \Sigma)$, the differential entropy and clarity are given by $h[Y] = \log \sqrt{(2\pi e)^n |\Sigma|}$ and $q[Y] = \frac{1}{1 + |\Sigma|}$.

\subsection{Spatiotemporal Gaussian Process Modeling}
We model the spatiotemporal field $f : \mathbb{R}_{\geq 0} \times \mathcal{D} \to \mathbb{R}$ using Gaussian Process Regression (GPR), where $f(t, p)$ denotes the field value at time $t$ and location $p \in \mathcal{D} \subset \mathbb{R}^D$. The observations are modeled as
\eqn{
  y(t, p) = f(t, p) + \epsilon, \quad \epsilon \sim \mathcal{N}(0, \sigma^2_m),
}
with GP prior
\eqn{
  f(t, p) \sim \mathcal{GP}(0, k_{\Phi}(t, t', p, p')),
}
where $k_{\Phi}$ is the kernel with hyperparameters $\Phi$. The posterior at test point $(t_*, p_*)$ is Gaussian with
\begin{subequations}
\begin{align}
  \mu &= k_{f_*}^\top (K_{yy})^{-1} y, \\
  v &= k_{\Phi}(t_*, t_*, p_*, p_*) - k_{f_*}^\top (K_{yy})^{-1} k_{f_*},
\end{align}
\end{subequations}
where $K_{yy} = K_{ff} + \sigma^2_m I$, and $k_{f_*}$ contains kernel evaluations with the test point. Hyperparameters $\Phi$ are learned by maximizing the marginal likelihood \cite{edward2006rasmussen}. To reduce the $\mathcal{O}(N^3)$ training and $\mathcal{O}(N^2)$ prediction cost of GPR, we use a state-space formulation \cite{kf_gp_3}, which recasts GP inference as Kalman filtering under the following assumption:
\begin{assumption}
\label{assumption:1}
The kernel is separable: $k_{\Phi}(t, t', p, p') = k_T(t, t') k_S(p, p')$, and the temporal kernel $k_T$ is isotropic.
\end{assumption}
The spatiotemporal field is represented using a set of $N_G$ grid points, $\Pcal_G = \{p_{g,i}\}_{i=1}^{N_G} \subset \Dcal$, over the spatial domain. The state at each point is governed by an SDE, which become spatially correlated using a spatial kernel. The state of the environment at the $i^{th}$ grid point is given as
\begin{subequations}
\begin{align}
  ds_i(t) &= A s_i(t) dt + B dW_i(t), \\
  z_i(t) &= C s_i(t), \quad s_i(0) \sim \mathcal{N}(0, \Sigma), \\
    \mathbf{f}(t) &= \sqrt{K_{GG}} (I_{N_G} \otimes C) \mathbf{s},
\end{align}
\end{subequations}
where $\mathbf{f}(t) = [f(t, p_{g,1}), \dots, f(t, p_{g,N_G})]^\top \in \mathbb{R}^{N_G}$ represents the value of spatiotemporal field at each grid point, and $\mathbf{s} = [s_1^\top, \dots, s_{N_G}^\top]^\top \in \mathbb{R}^{n_k N_G}$ is the stacked vector representing the state at each grid point. Matrices $A$, $B$, $C$ depend on $k_T$, and $\Sigma \in \mathbb{S}_{++}^{n_k}$ satisfies $A\Sigma + \Sigma A^\top = -BB^\top$. The spatial correlation is encoded in $K_{GG}$ with $[K_{GG}]_{ij} = k_S(p_{g,i}, p_{g,j})$.  We refer to this modeling approach as a spatiotemporal Gaussian process Kalman filtering (STGPKF).

\subsection{Clarity Dynamics for the Spatiotemporal Field}
\label{sec:clarity_dynamics_prelim}
To handle measurements at off-grid locations,~\cite{agrawal2024multi} adjusts the Kalman filter update to account for the mismatch with grid points. Consequently, at a fixed $p \in \mathcal{D}$:
\begin{subequations}
\begin{align}
  \dot{s} &= A s + B w, \quad w(t) \sim \mathcal{N}(0, I), \\
  y &= L s + v, \quad v(t) \sim \mathcal{N}(0, V \Delta T),
\end{align}
\end{subequations}
where $s \in \mathbb{R}^{n_k}$ is the environment state, $r = \Psi(x)$ is the robot position, where $\Psi(\cdot)$ returns robot position, and
\eqn{
  L = \frac{k_S(r, p)}{k_S(p, p)} C, \quad V = \sigma_m^2 + k_S(r, r) - \frac{k_S^2(r, p)}{k_S(p, p)}.
}
The KF covariance $\Sigma$ evolves as
\eqn{
  \dot{\Sigma} = A\Sigma + \Sigma A^\top + BB^\top - \Sigma L^\top (V \Delta t)^{-1} L \Sigma.
}
Let $\Pi = C\Sigma C^\top$ be the variance of $f(t, p)$, and define clarity as $q = 1/(1 + \Pi)$. The clarity dynamics are given as $
\dot{q} = -q^2 C \dot{\Sigma} C^\top$. This expression can be simplified to: 
\eqn{
\label{clarity_dynamics}
  \dot{q} = S(x, p)(1 - q)^2 - D(p, q),
}
where $S(x, p)$ captures clarity gain at $p$ from a measurement at $x$, and $D(p, q)$ models clarity decay at $p$ \cite{agrawal2024multi}. For Matérn-1/2 temporal kernel:
\begin{align}
    S(x, p) &= \frac{1}{\Delta T} \cdot \frac{k_S(r, p)^2}{k_S(p, p)\left(k_S(r, r) + \sigma_m^2\right) - k_S(r, p)^2}, \\
    D(p, q) &= 2 \lambda_t \left( \left( \sigma_t^2 + 1 \right) q^2 - q \right),\label{eq:clar_dyn_decay}
\end{align}
where $k_S(r,p) = \sigma_s^2 \exp\left(-\|r-p\|/\lambda_s\right)$ is the Matérn-1/2 spatial kernel, $\lambda_t,\lambda_s$ are temporal and spatial lengthscales, and $\sigma_t,\sigma_s$ are variance parameters controlling function scale. 

\begin{figure*}[t]
  \centering
   \captionsetup{font=footnotesize}
  \includegraphics[width=2.05\columnwidth]{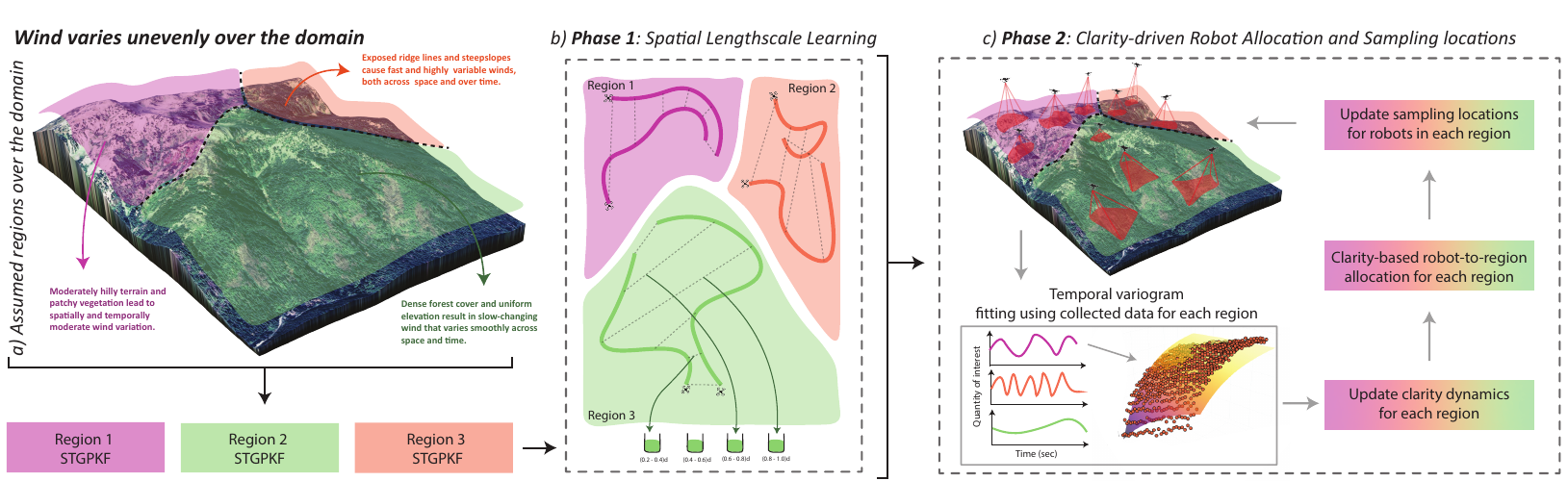}
  \caption{\textbf{Motivation}: A non-uniform spatiotemporal environment where wind varies slowly in forests (high lengthscales) and rapidly on ridgelines (low lengthscales). Robots aim to estimate this field. \textbf{The proposed framework}: (a) The environment is partitioned into regions, each modeled as an independent spatiotemporal Gaussian process Kalman filter (STGPKF) with constant spatial and temporal lengthscales. (b) In Phase 1, robots collect co-temporal spatial data to estimate region-specific spatial lengthscales via variogram fitting. (c) In Phase 2, robots are assigned to regions and fixed sampling locations based on initial clarity dynamics. Temporal lengthscales are updated at regular intervals, and allocation is updated based on the updated STGPKF model.}
  \vspace{-17pt}
  \label{fig:mot_prop}
\end{figure*}

\subsection{Variograms}
\label{sec:variogram_prelim}
Variograms provide a non-parametric way to estimate the hyperparameters of stationary Gaussian Processes (GPs), particularly the lengthscale. Given a set of scalar field measurements $\mathbb{D} = \{(a_i, b_i)\}_{i=1}^N$, where $a_i \in \mathbb{R}^d$ are input locations and $b_i \in \mathbb{R}$ are field values:

\begin{definition}[Theoretical Variogram {\cite{variograms}, Eq.~7.6}]
For a zero-mean second-order stationary field $Z: \mathbb{R}^d \to \mathbb{R}$, the theoretical variogram is defined as
\eqn{
  \gamma(d) = \frac{1}{2} \mathbb{E}[(Z(a_i) - Z(a_j))^2], \;\; \text{where } d = \|a_i - a_j\|.
}
\end{definition}

\begin{definition}[Empirical Variogram]
The empirical variogram at distance $d$ is estimated from data as
\eqn{
\label{eq:empirical_variogram}
  \hat\gamma(d) = \frac{1}{2|N(d)|} \sum_{(i,j) \in N(d)} (b_i - b_j)^2,
}
where $N(d)$ is the set of index pairs $(i, j)$ such that $i \neq j$ and $\|a_i - a_j\| \in (d - \epsilon, d + \epsilon)$ for a small bin width $\epsilon$.
\end{definition}

The kernel lengthscale is estimated by fitting a theoretical variogram model (e.g., exponential, Matérn) to the empirical variogram using nonlinear least squares.

\section{Problem Motivation \& Formulation}\label{sec:prob}
\subsection{Overall Problem Motivation and Formulation}

Consider a set of robot indices $\Rcal$, with $N_r = |\Rcal|$. Let $\Tcal$ denote the set of region indices, with $M = |\Tcal|$, and let $\Dcal \subset \mathbb{R}^D$ denote the domain. 
The domain is partitioned into disjoint regions $\{\Dcal_j\}_{j \in \Tcal}$, i.e., $\Dcal = \bigcup_{j \in \Tcal} \Dcal_j$ with $\Dcal_{j_1} \cap \Dcal_{j_2} = \emptyset$ for all $j_1, j_2 \in \Tcal$, $j_1 \neq j_2$. 
For compactness, we will refer to a region simply as $j \in \Tcal$ and a robot as $i \in \Rcal$.

\begin{assumption}
\label{assumption:2}
Within each region $j \in \Tcal$, the spatiotemporal field is second-order stationary, i.e., $E[f(p,t)] = \mu$ and $\mathrm{Cov}(f(p,t), f(p+\Delta p, t+\Delta t)) = \mathrm{Cov}(\Delta p, \Delta t)$. Each region is characterized by a constant spatial and temporal lengthscale, which may differ across regions.
\end{assumption}

For example, in~\cref{fig:mot_prop}, wind varies rapidly on ridgelines (small lengthscales) and slowly in dense forests (large lengthscales), with intermediate variation in hilly terrain. 

\renewcommand{\theoverprob}{Problem}
\begin{overprob}
\label{prob_1}
Let $f(t, p) : \mathbb{R}_{\geq 0} \times \mathcal{D} \to \mathbb{R}$ denote the field value at time $t$ and location $p \in \mathcal{D}$. Robots sense the field at discrete time steps $t_k$, $k \in \mathbb{Z}$. At each step $t_k$, every robot $i \in \mathcal{R}$ selects a spatiotemporal sampling location $X_k^i = (t_k^i, p_k^i)\in \mathbb{R}^{ (1+D)}$ and collects the corresponding field measurements $y_k^i \in \mathbb{R}$. 

Let $\mathbb{D}_k = \bigcup_{\ell = 1}^{k-1} \bigcup_{i \in \mathcal{R}} \{(X_\ell^i, y_\ell^i)\}$ 
denote the set of all data collected by all robots before iteration $k$. At iteration $k$, the $\mathbb{D}_k$ is used to update a single shared probabilistic model of the field. This model provides the posterior mean $\mu_k(t, p)$ and predictive covariance $\Sigma_k(t, p)$ for any test input $(t, p)$. We adopt \textit{clarity}, an uncertainty measure defined as
\eqn{
q_k(t, p) = \frac{1}{1 + |\Sigma_k(t, p)|}, \quad q_k(t, p) \in (0, 1).
}

At each planning step, robots choose sampling locations to maximize the average clarity as\footnote{
If the ground truth is known offline, the objective can be to directly minimize prediction error. In the more realistic case of unknown ground truth, we optimize a surrogate information-theoretic metric.}
\eqn{
\argmax_{\{X_k^i\}_{i \in \mathcal{R}}} \; \frac{1}{|\mathcal{D}|} \int_{\mathcal{D}} q_k(t, p) \, dp.
}
\end{overprob}

\subsection{Method Overview and Problem Reformulation}

In~\cref{prob_1}, clarity reflects prediction uncertainty but is only meaningful when the GP model is well-calibrated, which depends on accurate spatial and temporal lengthscales. Solving~\cref{prob_1} is challenging for two key reasons. First, sampling decisions and lengthscale estimates are coupled: sample locations affect lengthscale estimation, while the estimated lengthscales shape the GP posterior and, in turn, the clarity metric used for planning. Second, non-uniform environments exhibit region-dependent spatial and temporal correlations, requiring distinct lengthscales per region. To address these challenges, we propose a two-phase framework in which each region is modeled independently using a stationary GP with its own spatial and temporal lengthscales.

In \textbf{Phase 1}, robots use the proposed adaptive planner to collect samples for estimating region-specific spatial lengthscales. For each $j \in \mathcal{T}$, we estimate the spatial lengthscale $\lambda_{s,j}$ using a variogram-based method (\cref{sec:variogram_prelim}). In \textbf{Phase 2}, robots are assigned fixed sampling locations within each region to collect temporally dense data for iteratively updating the temporal lengthscale $\lambda_{t,j}$. At fixed intervals, the GP for region $j$ is reconstructed using its fixed $\lambda_{s,j}$ and updated $\lambda_{t,j}$. All planning and communication is centralized.

To quantify the uncertainty in variogram-based lengthscale estimation, we use results from the nonlinear least squares (NLS) literature \cite{estimation_book}. Subject to the assumptions stated in \cref{lemma:variance_exact}, we establish the variance of the variogram-based estimator for general kernel parameters $\rho = (\sigma, \lambda)$, where $\sigma$ is the variance and $\lambda$ is the lengthscale. Later, we distinguish spatial and temporal cases by writing $\rho_s = (\sigma_s, \lambda_s)$ and $\rho_t = (\sigma_t, \lambda_t)$.

\begin{lemma}
\label{lemma:variance_exact}
Let $\hat{\rho} = (\hat{\sigma}, \hat{\lambda})$ denote the estimated variogram parameters, obtained by fitting a theoretical variogram model $\gamma(h;\rho)$ to the empirical semivariances $\hat{\gamma}(h_b)$, where the latter are computed over $N_b$ spatial lag bins indexed by $b \in \Bcal$ with $N_b = |\Bcal|$. Each bin $b$ corresponds to a spatial lag interval $h_b = [(b-1)\epsilon,\; b\epsilon) \subset \mathbb{R}$, where $\epsilon > 0$ is the bin width.  Assume the empirical residuals are uncorrelated across bins and define  $\Sigma = \mathrm{diag}\!\left(\sigma_b^2 /{n_b}\right)$ and $J_b = -\left.\frac{\partial}{\partial \rho} \gamma(h_b; \rho)\right|_{\rho = \rho_0}$ , 
where $\sigma_b^2$ is the sample variance of squared differences in bin $b$, and $n_b$ is the number of measurement pairs in $h_b$. Let $\textbf{J} = [J_1, \dots, J_{N_b}]^\top$, i.e., a matrix whose $b$-th row is $J_b$. Then, under a first-order linearization of the least squares estimator, the variance is
\eqn{
\mathrm{Var}[\hat\rho]
= (\textbf{J}^\top \textbf{J})^{-1}
  \bigl(\textbf{J}^\top \,\Sigma\,\textbf{J}\bigr)
  (\textbf{J}^\top \textbf{J})^{-1}.
}
\end{lemma}

\begin{proof}
The proof follows from the NLS fitting \cite{estimation_book}. 
\end{proof}
We now decompose~\cref{prob_1} into two sub-problems. The first focuses on planning Phase 1 samples to minimize the spatial lengthscale estimation variance:
\renewcommand{\theprob}{1}
\begin{prob}
\label{prob_1a}
\textbf{(Spatial Lengthscale Learning)}.
At iteration $k$, each robot $i \in \Rcal$ selects one spatiotemporal location $X_k^i = (t_k^i, p_k^i) \in \mathbb{R}^{1+D}$. 
The goal is to learn the spatial lengthscale $\lambda_{s,j}$ in each region $j \in \Tcal$ by solving
\eqn{
\min_{\{X_k^i\}_{i \in \Rcal}} \; \mathrm{Var}[\lambda_{s,j}].
}
\end{prob}
Once $\lambda_{s,j}$ is learned for each region $j$, we construct a region-specific STGPKF model and define the clarity dynamics using~\eqref{clarity_dynamics}. Since $\lambda_{t,j}$ is still unknown, we initialize it with a small default value to enable initial planning.

Using these initial clarity dynamics, robots are allocated to regions and assigned  sampling locations. Robots collect measurements at a fixed rate, and in each region, data from one robot is used to estimate $\lambda_{t,j}$ via variogram. The clarity dynamics are updated, and allocation is recomputed. To determine the allocation and sampling locations, we solve:

\renewcommand{\theprob}{2}
\begin{prob}
\label{prob_1b}
\textbf{(Robot Allocation)}.
The objective is to assign each robot $i \in \Rcal$ to a region $j \in \Tcal$ and compute its sampling location. 
At iteration $k$, the optimization is
\begin{subequations}
\label{eq:prob_1b}
\begin{align}
    \max_{x_{ij}, X_k^i} \quad & \sum_{i} \sum_{j} x_{ij} \int_{\Dcal_j} \bigl[ q(t_k,p) - q(t_{k+1},p;\mathbf{X}_k^i) \bigr] dp 
    \label{eq:prob_1b_obj} \\
    \textrm{s.t.} \quad & \sum_{i} x_{ij} \geq 1 \quad \forall j \label{eq:prob_1b_atleast_1r} \\ 
    & \sum_{j} x_{ij} = 1 \quad \forall i \label{eq:prob_1b_exactly_1t} 
\end{align}
\end{subequations}
where $x_{ij}$ is a binary variable indicating whether robot $i$ is assigned to region $j$, and $X_k^i$ is the sampling configuration of robot $i$ at iteration $k$. 
The integral captures clarity gain over region $\Dcal_j$. 
Constraints ensure each region has at least one robot and each robot is assigned to exactly one region.
\end{prob}

\section{Phase I: Adaptive Spatial Lengthscale Estimation via Variogram Sampling}\label{sec:method_prob1}
This section presents the proposed solution to the ~\cref{prob_1a}. To estimate the spatial lengthscale $\lambda_{s,j}$ in each region $j \in \mathcal{T}$, we design a data collection strategy that emphasizes spatial diversity while minimizing estimation variance. Each region is assigned a pair of robots that perform \textit{concurrent sampling}—i.e., they collect measurements simultaneously at spatially separated locations. This is crucial to ensure that observed differences reflect spatial, not temporal, variation. For readability, we describe the algorithm for a single region $j$, dropping the index $j$ from the equations. Previously, we defined the dataset $\mathbb{D}_k = \bigcup_{\ell = 1}^{k-1} \bigcup_{i \in \mathcal{R}} \{(X_\ell^i, y_\ell^i)\}$
as the collection of all spatiotemporal samples gathered by the robot team up to (but not including) iteration $k$.  For this phase, we only have two robots in each region. The \textit{spatial lag} of such a pair is defined as the Euclidean distance between the two measurement locations.

To systematically organize these lags, we divide the possible range of spatial separations into discrete intervals called \textit{bins}. Each bin $b \in \mathcal{B}$ corresponds to a spatial lag interval $h_b = [(b - 1)\epsilon,\; b\epsilon) \subset \mathbb{R}$, 
where $\epsilon > 0$ is the bin width. We define the dataset $\mathbb{D}_k^b$, the set of all concurrent measurement pairs—accumulated up to iteration $k-1$, whose spatial separation falls within the $b^{\text{th}}$ bin as
\eqn{
\mathbb{D}_k^b = \bigcup_{\ell = 1}^{k-1} \left\{ \left(y_\ell^{1}, y_\ell^{2}\right) \;\middle|\; \left\|p_\ell^{1} - p_\ell^{2}\right\| \in h_b,\; t_\ell^{1} = t_\ell^{2} \right\}.
}
Thus, the dataset used for variogram estimation up to iteration $k$ can be expressed as the union over all bins as
$\mathbb{D}_k = \bigcup_{b \in \mathcal{B}} \mathbb{D}_k^b$. Now, as shown in \cref{sec:variogram_prelim}, the empirical and theoretical variograms are computed. 
 The empirical variogram $\hat{\gamma}_j(h_b)$ is computed using all such bins via~\eqref{eq:empirical_variogram}. The parameters $\hat{\rho}_{s,j} = (\hat{\sigma}_{s,j}, \hat{\lambda}_{s,j})$ of a theoretical variogram model $\gamma(h; \rho_j)$ (e.g., Matérn-1/2) are fit using NLS.

To guide adaptive sampling, we design a bin-based planner that allocates samples to spatial lag bins to minimize the variance of the estimated spatial lengthscale $\hat{\lambda}_{s,j}$. Crucially, this cannot be achieved without sufficient coverage across all lag bins, as reliable variogram fitting requires sample pairs spanning a range of spatial separations. The planner, therefore, prioritizes bins that either contribute significantly to the variance or are undersampled. Based on the total variance expression in~\cref{lemma:variance_exact}, we compute the per-bin contribution to the variance as
\begin{equation}
\label{eq:bin_contrib}
C_b = J_{b}^2 \frac{\sigma_{b}^2}{|\mathbb{D}^{b}_{k}|}.
\end{equation}
where $J_b$ is the sensitivity term for bin $b$, and $\sigma_b^2$ is the empirical variance of squared differences in that bin.

To determine which bin to sample next, we compute a \textit{score} $S_b$ for each bin that balances two objectives: reducing variance and ensuring coverage. The score is defined as
\begin{equation}
\label{eq:bin_score}
S_b = \frac{|\mathbb{D}_k^b|}{B_d} \psi(C_b) + \left(1 - \frac{|\mathbb{D}_k^b|}{B_d}\right) \psi(G_b).
\end{equation}
Here, $B_d$ is the total sample budget, $\psi(\cdot)$ denotes the L1 normalization (see~\cref{sec:prelim}), and $G_b = \frac{1}{|\mathbb{D}_k^b| + \delta}$ promotes coverage, with $\delta > 0$ ensuring numerical stability. The first term emphasizes bins with high variance contribution; the second favors under-sampled bins.
We convert the scores into a sampling distribution using a softmax function
\begin{equation}
w_b = \frac{\exp(S_b / T)}{\sum_{b'} \exp(S_{b'} / T)}.
\end{equation}
where $T > 0$ is a temperature parameter controlling distribution sharpness. A bin $b^*$ is sampled according to $w_b$, and a concurrent pair $(p_1, p_2)$ is selected with spatial separation close to $h_{b^*}$. To preserve independence assumptions for valid variogram fitting, both points must be sufficiently far from prior samples in $\mathbb{D}_k^{b^*}$. This adaptive process continues until the sample budget $B_d$ is exhausted.

\section{Phase II: Clarity-Aware Robot Allocation and Sampling}\label{sec:method_prob2}
Solving \cref{prob_1b}, however, is more challenging. The problem is NP-hard and involves a tightly coupled task (region) allocation and sampling location selection, leading to a computationally expensive MINLP. The nonlinearity arises from the integral term in the objective, which quantifies the clarity gain based on robot sampling locations. 

To improve tractability, we approximate the solution by reformulating \cref{prob_1b} as an MILP. The problem is tackled in three steps. First, we precompute candidate sampling locations using Voronoi partitioning, where each region is divided into Voronoi cells, whose centroids serve as sampling locations. For each possible number of robot assignments to a region (from 1 to $N_r - M$), we compute a Voronoi partition and use its centroids as sampling locations. For each region $j \in \Tcal$, we select one partition from a set of all possible Voronoi partitions, defined as
\begin{equation}
\Vcal_j = \{ V_{j,n_r} \mid n_r \in \Rcal_{\max} \},
\end{equation}
where $V_{j,n_r}$ is a Voronoi partition in region $j$ with $n_r$ robots, and $\Rcal_{\max} = \{1, \dots, N_r - M\}$ ensures at most $N_r - M$ robots per region. 
Each candidate Voronoi partition \( V_{j,n_r} \) has an associated set of centroids:
\begin{equation}
\Ccal_{j,n_r} = \{ c_{j,n_r,n_c} \mid n_c = 1, \dots, n_r \},
\end{equation}
where \( c_{j,n_r,n_c} \) is the $n_c$-th centroid in region $j$ under assignment $n_r$. Each centroid corresponds to a sampling location when that assignment is made.

Next, we compute the time required to increase clarity in each region $j \in \Tcal$, for all possible robot allocations. Each region contains a subset of grid points $\Pcal_{G,j} \subset \Pcal_G$, and we estimate the time required to increase clarity at each grid point $p \in \Pcal_{G,j}$ using the current clarity state $q_j(t, p)$ and a desired target $\bar{q}(p)$. This evolution is driven by a sampling plan defined by the centroids $\Ccal_{j, n_r}$. The time-to-clarity function is defined as
\begin{equation}
T_{j,n_r}(t, \Pcal_{G,j}) = \max_{p \in \Pcal_{G,j}} T\big(q_j(t, p), \bar{q}(p) \mid \Ccal_{j,n_r}, \dot{q}_j\big).
\label{eq:time_to_clarity}
\end{equation}
Here, $T(\cdot)$ denotes the estimated time required for the clarity at location $p \in \Pcal_{G,j}$ to reach the desired threshold under the influence of sampling from centroids in $\Ccal_{j,n_r}$. The maximum ensures that the slowest-improving point dominates the regional completion time.

Then, using the values of $T_{j,n_r}$, we define the following MILP to minimize the total time and assignment cost:
\begin{subequations}
\label{eq:task_allocation_prob}
\begin{align}
\min_{x_{ij}, n_{jr}} \quad & \sum_{j \in \Tcal} T_{j,n_r} n_{jr} + \sum_{j \in \Tcal} \sum_{i \in \Rcal} M(i,j) x_{ij} \label{eq:milp_obj_a} \\
\textrm{s.t.} \quad & \sum_{i \in \Rcal} x_{ij} \geq 1 \quad \forall j \in \Tcal \label{eq:milp_con_a}, \\
& \sum_{j \in \Tcal} x_{ij} = 1 \quad \forall i \in \Rcal \label{eq:milp_con_b},
\end{align}
\end{subequations}
where, $x_{ij} \in \{0,1\}$ is a binary decision variable indicating whether robot $i$ is assigned to region $j$, and $n_{jr}$ is the number of robots assigned to region $j$. The term $M(i,j)$ denotes the cost of assigning robot $i$ to region $j$, while $T_{j,n_r}$ represents the time required to achieve the desired clarity when $n_r$ robots are assigned. Constraints \eqref{eq:milp_con_a} and \eqref{eq:milp_con_b} ensure that each region is assigned at least one robot, and each robot is assigned exactly once.
The problem in \eqref{eq:task_allocation_prob} is a well-defined MILP and admits an optimal solution, provided that $T_{j,n_r}$ is non-increasing with respect to the number of robots $n_r$.

 When robots are assigned to a region, a linear assignment problem is solved to assign the $n_{jr}$  robots to the $n_{jr}$ possible sampling locations within that region:
\begin{subequations}
\label{eq:lin_assignment_prob}
\eqn{
    \textbf{X}_k^{i} = \min_{x_{ic}} \quad &   \sum_{c \in \Ccal_{j,n_r}}\sum_{i \in \Rcal} d_c(i,c \mid x_{ij}, n_{jr})x_{ic} \label{eq:lsp_obj_a} \\
    \textrm{s.t.} \quad  & \sum_{i \in \Rcal} x_{ic} = 1 \quad \forall c \in \Ccal_{j,n_r}, \label{eq:lsp_con_a}\\
    & \sum_{c \in \Ccal_{j,n_r}}  x_{ic} = 1 \quad \forall i \in \mathcal{R}, \label{eq:lsp_con_b} 
}
\end{subequations}
where $x_{ic}$ is a binary variable representing wether robot $i$ is assigned to centroid $c \in \Ccal_{j,n_r}$ of the Voronoi partition $V_{j,n_r}$. $d_c(i,c \mid x_{ij}, n_{jr})$ is the cost of moving $i^{th}$ robot current position to the centroid $c$ of the allocated region $j$. The result is the sampling locations for all robots at iteration $k$. Once the temporal lengthscale is re-estimated, the region allocation and sampling assignments are updated at the next iteration.

\section{Theoretical Guarantees}\label{sec:theo_guarantees}
This section presents theoretical guarantees: almost sure convergence of the spatial lengthscale estimate and a sublinear dynamic regret bound for clarity-based allocation.

\subsection{Almost‐Sure Convergence of the Spatial lengthscale $\lambda_s$}
\label{sec:thm_1}

In this section, we show that the estimated spatial kernel parameters $\hat \rho_s = (\hat \sigma_s, \hat \lambda)$ converge almost surely to the true parameters $\rho^*_s = (\sigma^*_s, \lambda^*_s)$.

Let \(n\) be the total number of concurrent sample pairs collected by the method described in \cref{sec:method_prob1}, and let
\eqn{
L_n(\rho_s)
=\sum_{b=1}^{N_b}\bigl[\hat\gamma_n(h_b)-\gamma(h_b;\rho_s)\bigr]^2,
}
where $\rho_s = (\sigma_s, \lambda_s)$ is the spatial kernel parameters, \(\hat\gamma_n(h_b)\) is the empirical semivariance computed from the \(n\) samples in bin \(b\).  Define the estimator
\eqn{
\hat\rho_{s,n}
=\arg\min_{\rho_s\in\Lambda}L_n(\rho_{s}).
}
We show that, under mild assumptions, the estimator sequence ${\hat\rho_{s,n}}$ converges almost surely to the true variogram parameters $\rho^*_s$ as the sample budget $B_d \to \infty$.

\begin{assumption}\label{asm:semivar}
For each lag‐bin \(b\), the empirical semivariance \(\hat\gamma(h_b)\) defined in \eqref{eq:empirical_variogram} satisfies
$\hat\gamma(h_b)
=\gamma\bigl(h_b;\rho^*_s\bigr)\;+\;\varepsilon_b,$
with \(\E[\varepsilon_b]=0\) and \(\mathrm{Var}(\varepsilon_b)<\infty\).  Moreover, under~\cite{cressie2015statistics}, 
\(\E[\hat\gamma(h_b)]-\gamma(h_b)=\Ocal(1/n_b)\) and \(\mathrm{Var}(\varepsilon_b)=\Ocal(1/n_b)\), where \(n_b=|N(h_b)|\).
\end{assumption}

\begin{assumption}\label{asm:bins}
As the total number of concurrent samples \(n\to\infty\), each bin’s sample count \(n_b\to\infty\). Equivalently, there exists \(c_b>0\) such that \(n_b/n\to c_b\)
\end{assumption}

\begin{assumption}\label{asm:compact_param}
The parameter space \( \Lambda \subset \mathbb{R}^d \) is compact; that is, it is closed and bounded.
\end{assumption}

\begin{assumption}\label{asm:ident}
The map \(\rho_s\mapsto\bigl(\gamma(h_1;\rho_s),\dots,\gamma(h_{N_b};\rho_s)\bigr)\) is injective on the compact parameter set \(\Lambda\).  In particular, \(\rho_{s_1}\neq\rho_{s_2}\) implies \(\gamma(\,\cdot\,;\rho_{s_1})\neq\gamma(\,\cdot\,;\rho_{s_2})\).
\end{assumption}

\begin{thm}\label{prop:consistency}
Under Assumptions \ref{assumption:2}, \ref{asm:semivar}, \ref{asm:bins}, \ref{asm:compact_param}, and \ref{asm:ident}, along with \cref{lemma:variance_exact}, the estimator \(\hat\rho_s\) satisfies
\eqn{
\hat\rho_s \;\xrightarrow{\mathrm{a.s.}}\;\rho^*_s
\quad\text{and hence}\quad
\hat\lambda_s \;\xrightarrow{\mathrm{a.s.}}\;\lambda_s^*
\quad\text{as }n\to\infty.
}
\end{thm}

\begin{proof}
The proof is constructed as follows. As more concurrent sample pairs are collected, empirical semivariances become more accurate. Under standard assumptions—unbiased noise, sufficient bin coverage, and a compact parameter space—the empirical loss converges uniformly to its expected value, uniquely minimized at the true parameter. This ensures the estimator converges almost surely to the true spatial lengthscale. Formal proof is given below:

Define the loss as
\begin{subequations}
\eqn{
L(\rho_s)\;&=\;\sum_{b=1}^{N_b}\E\bigl[\hat\gamma(h_b)-\gamma(h_b;\rho_s)\bigr]^2 \\
        &=\sum_{b=1}^{N_b}\bigl[\gamma(h_b;\rho_s^*)-\gamma(h_b;\rho_s)\bigr]^2  
        \;+\;\sum_{b=1}^{N_b}\mathrm{Var}(\varepsilon_b).
}
\end{subequations}

Since $\mathrm{Var}(\varepsilon_b)$ is independent of $\rho_s$, \cref{asm:semivar} implies that $L(\rho_s)$ is uniquely minimized at the true parameter $\rho_s^*$ by identifiability, and furthermore specifies that for each lag bin $b$,
\eqn{
  \hat\gamma(h_b)
  = \gamma(h_b;\rho_s^*) + \varepsilon_b,
  \quad \E[\varepsilon_b] = 0,
  \quad \mathrm{Var}(\varepsilon_b) < \infty.
}
For any $\rho_s\in\Lambda$, define the \emph{residual}
\eqn{
  r_b(\rho_s)
  \;=\;\hat\gamma(h_b) \;-\;\gamma(h_b;\rho_s)
  = \bigl[\gamma(h_b;\rho_s^*)-\gamma(h_b;\rho_s)\bigr] \;+\;\varepsilon_b.
}
Because
1) $\Lambda$ is compact per \cref{asm:compact_param} and $\gamma(h;\rho_s)$ is continuous in $\rho_s$, the deterministic term 
   $\gamma(h_b;\rho_s^*)-\gamma(h_b;\rho_s)$ 
   is uniformly bounded over $b$ and $\rho_s$,
2) and $\mathrm{Var}(\varepsilon_b)<\infty$ for every $b$, it follows that there exists a finite $M$, $\forall\,b,\;\forall\,\rho_s\in\Lambda$, such that
\eqn{
  \E\bigl[r_b(\rho_s)^2\bigr]
  \;=\;\bigl[\gamma(h_b;\rho_s^*)-\gamma(h_b;\rho_s)\bigr]^2
       + \mathrm{Var}(\varepsilon_b)
  \;\le\; M,
}
In other words, the residuals $r_b(\rho_s)$ have \emph{uniformly bounded second moments}.  Hence for each bin $b$, the Strong Law of Large Numbers applies to the i.i.d.\ sequence $\{r_b^{(k)}(\rho_s)\}_{k=1}^{n_b}$, giving
\eqn{
  \frac{1}{n_b}
  \sum_{(i,j)\in N(h_b)}
  \bigl[\hat\gamma(h_b)-\gamma(h_b;\rho_s)\bigr]^2
  \;\xrightarrow{\mathrm{a.s.}}\;
  \E\bigl[\hat\gamma(h_b)-\gamma(h_b;\rho_s)\bigr]^2.
}

For each bin \(b\).  \cref{asm:bins} ensures \(n_b/n\to c_b>0\), so that

\begin{equation}
\begin{aligned}
L_n(\rho_s)
&=\sum_{b=1}^{N_b}n_b
   \Bigl(\frac1{n_b}\sum_{(i,j)\in N(h_b)}
     \bigl[\hat\gamma(h_b)-\gamma(h_b;\rho_s)\bigr]^2\Bigr)\\
&\xrightarrow{\mathrm{a.s.}}
   \sum_{b=1}^{N_b}\bigl(n\,\tfrac{n_b}{n}\bigr)\,
     \E\bigl[\hat\gamma(h_b)-\gamma(h_b;\rho_s)\bigr]^2
   \quad(\tfrac{n_b}{n}\to c_b)\\
&=n\sum_{b=1}^{N_b}c_b\,
     \E\bigl[\hat\gamma(h_b)-\gamma(h_b;\rho_s)\bigr]^2
   =n\,L(\rho_s)\,
\end{aligned}
\label{eq:Ln-consistency}
\end{equation}

where $\sum_{b} c_b = 1$. Hence
\(\sup_{\rho_s\in\Lambda}\bigl|L_n(\rho_s)/n - L(\rho_s)\bigr|\to0\) almost surely, by finiteness of \(N_b\) and continuity of \(\gamma(h;\rho_s)\).

Finally, since \(\Lambda\) is compact and \(L(\rho_s)\) has a unique minimizer \(\rho_s^*\), the minimzers also converge almost surely \cite{van2000asymptotic}
\[
\hat\rho_s=\arg\min_\Lambda L_n(\rho_s)
\;\xrightarrow{\mathrm{a.s.}}\;
\arg\min_\Lambda L(\rho_s)=\rho_s^*.
\]
Projecting onto the second component yields \(\hat\lambda_s\to\lambda_s^*\).
\end{proof}

\subsection{Dynamic Regret Analysis of Phase II}

We analyze the performance of our clarity-aware robot allocation strategy under temporal lengthscale uncertainty. For readability, we use the notation $(k)$ to denote quantities at iteration $k$. At each iteration, the algorithm estimates the temporal lengthscale $\hat{\lambda}_{t,j}(k)$ in each region $j$ by refitting the variogram with newly collected data. This estimate determines the decay term in the clarity dynamics \eqref{eq:clar_dyn_decay}, which is then used to compute the predicted time-to-increase clarity $T_{j,n_j}(k)$ from \eqref{eq:time_to_clarity} for different allocations. The final assignment is selected by solving the MILP in \cref{sec:method_prob2}. 

Since $\hat{\lambda}_{t,j}(k)$ is learned from finite data, it typically underestimates the true value $\lambda_{t,j}^*$, leading to overestimated decay and inflated $T_{j,n_{jr}}(k)$. As a result, the algorithm may avoid assigning robots to regions where clarity would actually increase rapidly. The regret arises from this model mismatch—specifically, from not knowing the true temporal lengthscale parameter, which causes suboptimal decisions in the robot allocation step.

\subsubsection{Oracle vs. Proposed Algorithm}

Let $q_j(k)$ denote the clarity in region $j$ at iteration $k$. Let $n_{jr}(k)$ be the number of robots assigned to region $j$ at epoch $k$ by the algorithm, and $n_{jr}^*(k)$ the assignment chosen by an oracle with access to the true clarity dynamics. Define the clarity increase from epoch $k$ to $k+1$ as $\Delta q_j(k) = q_j(k+1) - q_j(k)$. Then, the instantaneous regret is defined as
\begin{equation}
r(k) = \sum_{j=1}^M \Delta q_j^{*}(k) - \sum_{j=1}^M \Delta q_j(k),
\end{equation}
where $\Delta q_j^{*}(k)$ denotes the clarity gain under the oracle assignment. The cumulative regret over $K$ iterations is:
\begin{equation}
R(K) = \sum_{k=1}^K r(k).
\end{equation}
Since the oracle assignment changes at each iteration $k$ due to the spatiotemporal nature of the environment, the regret is dynamic. To derive the regret, we use a lengthscale convergence result and some assumptions given below:

\begin{lemma}\label{lem:nls-rate}
For every region \(j \in \Tcal\), let \(N_j(k)\) denote the cumulative
number of \emph{co-temporal} measurements collected up to iteration \(k\).
The NLS estimator \(\hat\lambda_{t,j}(k)\) of the true
temporal length-scale \(\lambda^{\!*}_{t,j}\) satisfies
\eqn{
\mathbb{E}\bigl[\bigl|\hat\lambda_{t,j}(k)-\lambda^{*}_{t,j}\bigr|^2\bigr]
\;=\;\mathcal{O}\bigl(N_j(k)^{-1}\bigr), \qquad k \to \infty,
}
that is, its mean-square error decays at rate
\(N_j(k)^{-1}\).  

\end{lemma}

\begin{proof}
See \cite{estimation_book} for NLS asymptotic-normality result. 
\end{proof}

\begin{assumption}
\label{assump:clarity_lipschitz}
The clarity gain $\Delta q_j$ is Lipschitz continuous in the estimated temporal lengthscale. That is, there exists $L > 0$ such that for any allocation $n_j$,
\eqn{
|\Delta q_j(n_j; \hat{\lambda}_{t,j}) - \Delta q_j(n_j; \lambda_{t,j}^*)| \leq L \cdot |\hat{\lambda}_{t,j} - \lambda_{t,j}^*|.
}
\end{assumption}

\begin{assumption}
\label{assump:uniform_sampling}
At each iteration \( k \), region \( j \) receives at least \( c_k > 0 \) new samples, with \( c_k \geq c > 0 \) for all \( k \geq 1 \). Thus, the total number of samples satisfies \( N_j(k) \geq c_1 + \cdots + c_k \).
\end{assumption}

\begin{thm}
\label{thm:regret}
Under Lemma~\ref{lem:nls-rate}, Assumptions~\ref{assump:clarity_lipschitz}--\ref{assump:uniform_sampling}, the cumulative regret after $K$ iterations satisfies
\eqn{
R(K) = \mathcal{O}(\sqrt{K \log K}),
}
and the average regret per iteration vanishes: $R(K)/K \to 0$ as $K \to \infty$.
\end{thm}

\begin{proof} 
The proof is constructed as follows. First, since each region accrues at least one new concurrent sample per iteration, the NLS estimator obeys $\bigl|\hat\lambda_{t,j}(k)-\lambda_{t,j}^{*}\bigr|=\mathcal{O}(k^{-1/2})$ with high probability.  Uniform $L$-Lipschitz continuity of the clarity-gain function transfers this rate to the predicted gains.  Decomposing the instantaneous regret into an \emph{allocation gap} and an \emph{estimation gap} shows both are bounded by $L\bigl|\hat\lambda_{t,j}-\lambda_{t,j}^{\!*}\bigr|$, yielding $r(k)=\mathcal{O}\bigl(\sqrt{\log K}/\sqrt{k}\bigr)$.  Summing $\sum_{k=1}^{K}k^{-1/2}=\mathcal{O}(\sqrt{K})$ gives the cumulative bound $R(K)=\mathcal{O}\bigl(\sqrt{K\log K}\bigr)$, so the average regret $R(K)/K$ vanishes as $K\to\infty$.  The formal proof is given as follows:

By Lemma~\ref{lem:nls-rate}, the variance of the temporal lengthscale estimator satisfies
\eqn{
\operatorname{Var}[\hat{\lambda}_{t,j}(k)] = \mathcal{O}\left( \frac{1}{N_j(k)} \right),
}
which implies the standard deviation scales as
\eqn{
|\hat{\lambda}_{t,j}(k) - \lambda_{t,j}^*| = \mathcal{O}(N_j(k)^{-1/2}).
}
Now, using Assumption~\ref{assump:uniform_sampling}, we have
\eqn{
N_j(k) \geq c_1 + \cdots + c_k \geq c \cdot k,
}
for some constant \( c > 0 \). Therefore, we can write
\eqn{
|\hat{\lambda}_{t,j}(k) - \lambda_{t,j}^*| = \mathcal{O}(k^{-1/2}),
}
up to constants.

Then, by Assumption~\ref{assump:clarity_lipschitz}, we apply the Lipschitz property
\eqn{
\left| \Delta q_j(; \hat{\lambda}_{t,j}(k)) - \Delta q_j(; \lambda_{t,j}^*) \right| \leq L \cdot \mathcal{O}(k^{-1/2}) = \mathcal{O}(k^{-1/2}).\footnote{We write the bound as equality up to constants because the $\mathcal{O}(\cdot)$ notation absorbs constant multipliers such as the Lipschitz constant $L$.}
}

Now define the instantaneous regret at iteration $k$ as
\eqn{
r(k) := \sum_{j=1}^M \left(
\Delta q_j(n_j^*(k); \lambda_{t,j}^*) 
- \Delta q_j(n_{jr}(k); \hat{\lambda}_{t,j}(k)) \right).
}

Add and subtract $\Delta q_j(n_{jr}(k); \lambda_{t,j}^*)$ within each term:
\begin{align}
r&(k) = \sum_{j=1}^M \bigg[
\Delta q_j(n_j^*(k); \lambda_{t,j}^*) 
- \Delta q_j(n_{jr}(k); \lambda_{t,j}^*)\bigg] \nonumber\\
& + \sum_{j=1}^M \bigg[ 
\Delta q_j(n_{jr}(k); \lambda_{t,j}^*) 
- \Delta q_j(n_{jr}(k); \hat{\lambda}_{t,j}(k)) \bigg].
\end{align}

Apply the triangle inequality:
\begin{align}
|&r(k)| \leq \sum_{j=1}^M \left| 
\Delta q_j(n_j^*(k); \lambda_{t,j}^*) 
- \Delta q_j(n_{jr}(k); \lambda_{t,j}^*)\right| \nonumber \\
& + \sum_{j=1}^M \left| \Delta q_j(n_{jr}(k); \lambda_{t,j}^*) 
- \Delta q_j(n_{jr}(k); \hat{\lambda}_{t,j}(k))\right|.
\end{align}

Because \(n_j^{*}(k)\) maximises
\(\Delta q_j(\cdot;\lambda^{\!*}_{t,j})\),
the first term is non-negative.
Add and subtract
\(\Delta q_j(n_j^{\!*}(k);\hat\lambda_{t,j}(k))\) and invoke the Lipschitz
property (\cref{assump:clarity_lipschitz}) to obtain
\begin{align}
0\;&\le\;
\Delta q_j\bigl(n_j^{\!*};\lambda^{*}\bigr)
-\Delta q_j\bigl(n_{jr};\lambda^{*}\bigr) \\
&\le
L\bigl|\hat\lambda_{t,j}(k)-\lambda^{\!*}_{t,j}\bigr|
= \mathcal{O}\bigl(k^{-1/2}\bigr).
\end{align}
The same bound holds for the estimation gap. Therefore,
\eqn{
r(k) \leq \sum_{j=1}^M \mathcal{O}(k^{-1/2}) = \mathcal{O}(k^{-1/2}).
}

Summing over $k = 1$ to $K$, the cumulative regret is
\eqn{
R(K) = \sum_{k=1}^K r(k) = \sum_{k=1}^K \mathcal{O}(k^{-1/2}).
}
We bound this sum using the integral:
\eqn{
\sum_{k=1}^K \frac{1}{\sqrt{k}} \leq \int_1^K \frac{1}{\sqrt{x}}\, dx + 1 = 2\sqrt{K} - 2 + 1 = \mathcal{O}(\sqrt{K}).
}

However, the bound $r(k) = \mathcal{O}(k^{-1/2})$ only holds with high probability (not uniformly) at each time step $k$ and for each region $j$. More precisely, define the failure event:
\eqn{
\text{failure}_{k,j} := \left\{ \left| \hat{\lambda}_{t,j}(k) - \lambda_{t,j}^* \right| > \epsilon_k \right\}.
}
Assume a concentration inequality of the form
\eqn{
\mathbb{P}\left( \left| \hat{\lambda}_{t,j}(k) - \lambda_{t,j}^* \right| > \epsilon_k \right) \leq \exp(-\alpha N_j(k) \epsilon_k^2),
}
where $\alpha$ is a positive constant. Set the failure probability at each $(k,j)$ to be at most $\delta'$, and solve for $\epsilon_k$:
\eqn{
\delta' = \exp(-\alpha N_j(k) \epsilon_k^2)
\quad \Rightarrow \quad 
\epsilon_k = \mathcal{O}\left( \sqrt{ \frac{ \log(1/\delta') }{ N_j(k) } } \right).
}
To ensure the bound holds \emph{uniformly} over all $K$ time steps and $M$ regions, apply Boole’s inequality:
\eqn{
\mathbb{P} \left( \bigcup_{k=1}^K \bigcup_{j=1}^M \text{failure}_{k,j} \right)
\leq KM \cdot \delta'.
}
To guarantee this probability is at most $\delta$, set $\delta' = \delta / (KM)$, yielding
\eqn{
\epsilon_k = \mathcal{O}\left( \sqrt{ \frac{ \log(KM/\delta) }{ N_j(k) } } \right)
= \mathcal{O}\left( \frac{ \sqrt{ \log(KM) } }{ \sqrt{k} } \right),
}
using $N_j(k) \geq c \cdot k$ from Assumption~\ref{assump:uniform_sampling}.

\begin{table}[ht]
\centering
\scriptsize
\begin{tabular}{c|c|c|c}
  \hline
  \makecell{\textbf{True} \\ $\lambda_s$} & \textbf{Method} 
                 & \makecell{\textbf{Estimated} \\ $\hat\lambda_s$ } 
                 & \makecell{\textbf{Mean Absolute} \\ \textbf{Percentage Error (\%)}} \\
  \hline
  \hline
  \multirow{2}{*}{0.267} & Greedy   & 0.403 ± 0.281 & 50.9\% \\
                         & \cellcolor{lightgray}Proposed & \cellcolor{lightgray}0.209 ± 0.019 & \cellcolor{lightgray}21.7\% \\
  \hline
  \multirow{2}{*}{0.333} & Greedy   & 0.389 ± 0.190 & 16.8\% \\
                         & \cellcolor{lightgray}Proposed & \cellcolor{lightgray}0.336 ± 0.115 & \cellcolor{lightgray}0.9\% \\
  \hline
  \multirow{2}{*}{0.500} & Greedy   & 0.551 ± 0.211 & 10.2\% \\
                         & \cellcolor{lightgray}Proposed & \cellcolor{lightgray}0.564 ± 0.169 & \cellcolor{lightgray}12.8\% \\
  \hline
  \multirow{2}{*}{0.667} & Greedy   & 0.552 ± 0.217 & 17.2\% \\
                         & \cellcolor{lightgray}Proposed & \cellcolor{lightgray}0.699 ± 0.131 & \cellcolor{lightgray}4.8\% \\
  \hline
  \multirow{2}{*}{0.833} & Greedy   & 0.689 ± 0.311 & 17.3\% \\
                         & \cellcolor{lightgray}Proposed & \cellcolor{lightgray}0.909 ± 0.153 & \cellcolor{lightgray}9.1\% \\
  \hline
\end{tabular}
\caption{Spatial lengthscale estimation results over 50 trials. Final column shows average percentage error $\left( \frac{|\hat\lambda_s - \lambda_s|}{\lambda_s} \times 100 \right)$.}
\label{tab:spatial_comparison}
\vspace{-17pt}
\end{table}
Thus, with probability at least $1 - \delta$, we have
\eqn{
\left| \hat{\lambda}_{t,j}(k) - \lambda_{t,j}^* \right| = \mathcal{O}\left( \frac{ \sqrt{\log(KM)} }{ \sqrt{k} } \right),
\quad \forall\, k,j.
}
By the Lipschitz continuity of the clarity dynamics, the per-step regret satisfies
\eqn{
r(k) = \mathcal{O}\left( \frac{ \sqrt{ \log(KM) } }{ \sqrt{k} } \right).
}

We now sum over $k = 1$ to $K$ to obtain the cumulative regret:
\begin{subequations}
\begin{align}
R(K)
&= \sum_{k=1}^K r(k)
= \sum_{k=1}^K \mathcal{O}\left( k^{-1/2} \cdot \sqrt{\log(KM)} \right) \\
&= \mathcal{O}\left( \sqrt{\log(KM)} \cdot \sum_{k=1}^K k^{-1/2} \right) \\
&= \mathcal{O}\left( \sqrt{K \log K} \right).
\end{align}

\end{subequations}

Finally, the average regret satisfies
\eqn{
\frac{R(K)}{K} = \mathcal{O}\left( \sqrt{ \frac{ \log(KM) }{K} } \right)
\to 0 \quad \text{as } K \to \infty.
}
\end{proof}

\section{Results \& Discussion}\label{sec:results}

This section validates our two-phase framework. We show that: (i) the adaptive variogram method accurately estimates region-specific spatial lengthscales, (ii) clarity-aware allocation improves robot distribution over time, and (iii) the full framework enhances field reconstruction. Implementation uses \textit{SpatiotemporalGPs.jl}~\cite{agrawal2024multi} for ground truth generation and \textit{DelaunayTriangulation.jl}~\cite{VandenHeuvel2024DelaunayTriangulation} for Voronoi partitioning.

\begin{figure*}[t]
  \centering
   \captionsetup{font=footnotesize}
  \includegraphics[width=2.00\columnwidth]{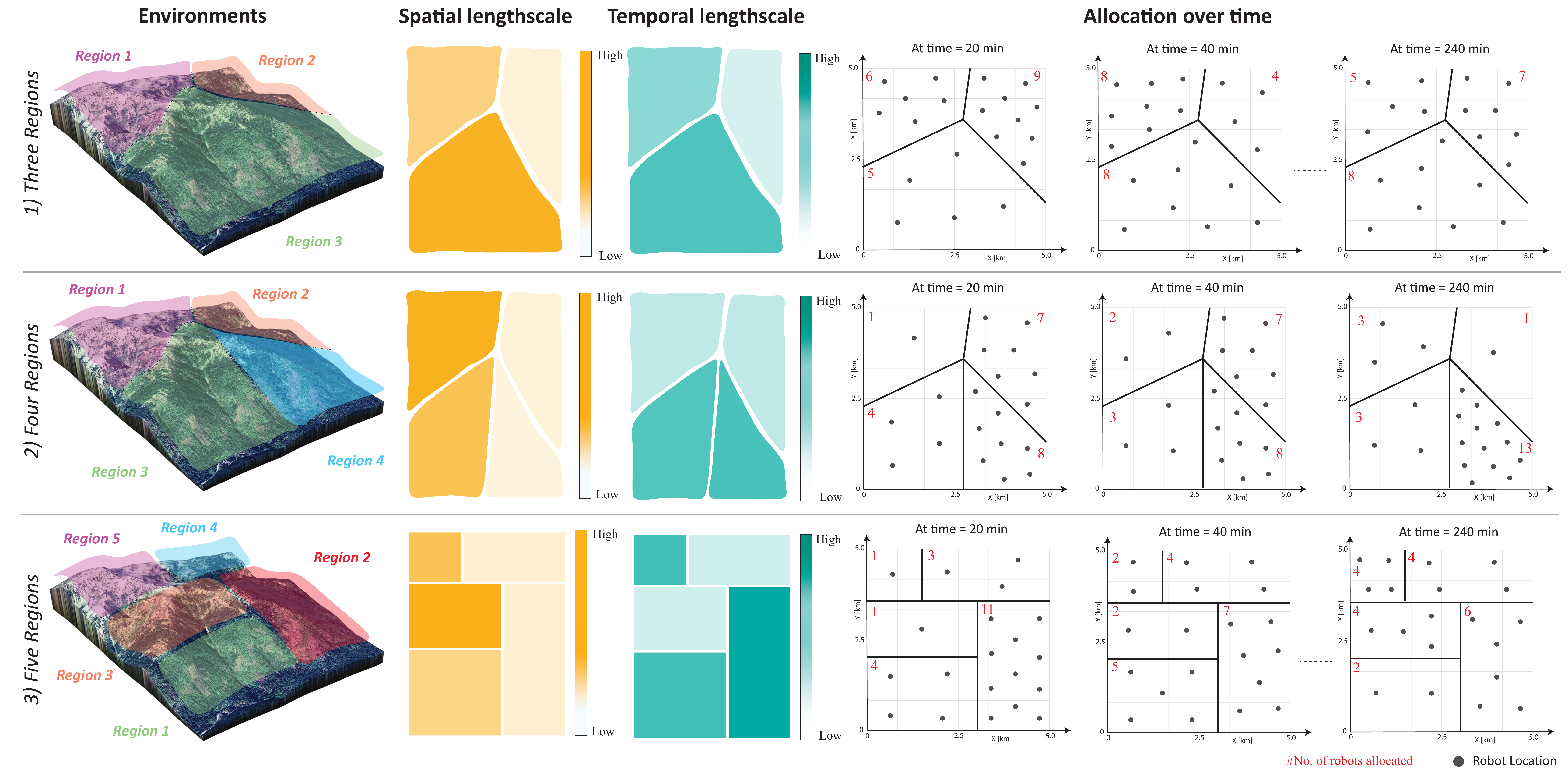}
  \caption{Robot allocation over time in three environments. Initial assignment reflects spatial lengthscales; later shifts adapt to clarity decay and learned temporal dynamics.}
  \vspace{-17pt}
  \label{fig:comp_allocation}
\end{figure*}

\subsection{Spatial Lengthscale Learning}

We compare our adaptive variogram-based sampling strategy to a greedy baseline that selects sample pairs with the largest residuals, neglecting spatial lag coverage. The goal is to estimate region-specific spatial lengthscales $\lambda_s$.

As shown in Table I, our method consistently outperforms the greedy baseline over 50 trials—particularly for small $\lambda_s$, where the lack of long-range samples in the greedy strategy leads to overestimation and high variance. For instance, at $\lambda_s = 0.267$, the greedy estimate is $0.403 \pm 0.281$, while our method achieves $0.209 \pm 0.019$.

By balancing variance reduction with spatial lag coverage, our method avoids overfitting to short-range structure—a common failure mode in greedy approaches. As $\lambda_s$ increases, both methods improve and the performance gap narrows, since smooth fields are easier to estimate and require less spatial diversity. Although \cref{lemma:variance_exact} provides an analytical expression for $\mathrm{Var}[\hat\lambda_s]$, it rests on idealized assumptions such as uncorrelated residuals and independent sample pairs. In practice, particularly in very smooth or rapidly varying fields, these assumptions may not hold—potentially causing the variance estimate to under- or overestimate true error. Thus, it serves primarily as a heuristic to guide adaptive sampling.

\subsection{Clarity-Aware Task Allocation Over Time}

\Cref{fig:comp_allocation} shows how robot allocation evolves across three environments with increasing spatial and temporal variability. Early allocation reflects spatial lengthscales and an initial assumption of fast temporal decay, leading to a relatively uniform distribution where regions with smaller spatial lengthscales (lower spatial correlation) receive more robots. As temporal lengthscales are learned, robots are reallocated based on clarity decay and saturation.

In Environment 1, Region 3 receives the most robots due to its large area, despite being smooth and slow-changing. In Environment 2, Region 1 starts with minimal allocation but gains robots once other regions are learned to decay more slowly,  highlighting the need for more frequent updates in Region 1 due to its smaller temporal lengthscale. Region 4 consistently receives many robots due to its extremely small spatial lengthscale, requiring dense sampling.
\begin{table}[ht]
\centering
\scriptsize
\begin{tabular}{c|c|c|c|c|c}
\hline
\textbf{Env} & \textbf{Method} 
& \textbf{Mean Clarity} & \textbf{RMSE} & \textbf{SMSE} & \textbf{NLPD} \\
\hline \hline
\multirow{2}{*}{1 (20)} 
& Baseline 1         &0.580    & 0.632     & 0.468    & 0.996     \\
& \cellcolor{lightgray} Proposed & \cellcolor{lightgray} 0.658 & \cellcolor{lightgray} 0.638 & \cellcolor{lightgray} 0.479 & \cellcolor{lightgray} 0.873 \\
\hline
\multirow{2}{*}{1 (50)} 
& Baseline 1         & 0.641    & 0.552     & 0.358    & 0.864     \\
& \cellcolor{lightgray} Proposed & \cellcolor{lightgray} 0.724 & \cellcolor{lightgray} 0.545 & \cellcolor{lightgray} 0.349 & \cellcolor{lightgray} 0.695 \\
\hline
\multirow{2}{*}{2 (20)} 
& Baseline 1         & 0.376 & 0.842  & 0.796    & 1.52      \\
& \cellcolor{lightgray} Proposed & \cellcolor{lightgray} 0.428 & \cellcolor{lightgray} 0.912 & \cellcolor{lightgray} 0.533 & \cellcolor{lightgray} 1.33 \\
\hline
\multirow{2}{*}{2 (50)} 
& Baseline 1         & 0.428   & 0.767 & 0.659    & 1.43      \\
& \cellcolor{lightgray} Proposed & \cellcolor{lightgray} 0.524 & \cellcolor{lightgray} 0.720 & \cellcolor{lightgray} 0.332 & \cellcolor{lightgray} 1.08 \\
\hline
\end{tabular}
\caption{Comparison of average clarity and reconstruction metrics across environments. Env(No.) represents the number of robots in the Env. The mean clarity of each region is shown.} 
\label{tab:estimate_comparison}
\vspace{-10pt}
\end{table}

In Environment 3, Region 2 is prioritized early due to low spatial correlation but loses robots after its slow temporal decay is learned. Region 3, with smaller spatial and temporal lengthscales, attracts more robots over time. Region 4 sees increased allocation later, as other regions saturate and its marginal clarity gain remains high. 

\subsection{Final Clarity and Field Estimation Accuracy}

\cref{tab:estimate_comparison} compares Baseline 1 and the proposed method across two environments using clarity, Root Mean Squared Error (RMSE), Standardized Mean Squared Error
 (SMSE), and Negative Log Predictive Density (NLPD). The proposed method achieves the best performance across all metrics. Baseline 1 assumes a globally small spatial lengthscale, ignoring spatial correlation and producing conservative estimates—especially in smooth regions where fewer samples would suffice. In contrast, the proposed method combines region-specific spatial modeling with clarity-aware allocation, focusing sensing where temporal change is rapid or spatial correlation is low. This results in a higher average clarity over regions, improved uncertainty quantification—reflected in lower NLPD—and consistent gains across all metrics.

\subsection{Hardware Demonstration}
We demonstrate real-time region allocation and reassignment using 10 Crazyflie 2.0 quadrotors, implemented with the Crazyswarm 1.0 ROS Noetic package. Three Crazyradio PA's handle communication, and 15 Vicon cameras provide localization. Goal positions are sent to the robots, which adjust their motion accordingly. Experiments in Environments 1 and 2 use the lengthscale configuration in~\cref{fig:comp_allocation}. The field evolves on a time scale of seconds to minutes, with region assignments updated every 10 seconds based on clarity dynamics. These results confirm that robots can track and adapt to changing assignments in real time. 

\section{Conclusion, Limitations, and Future Work}
We proposed a two-phase framework for multi-robot information gathering in non-uniform spatiotemporal environments. By estimating region-specific spatial and temporal lengthscales from concurrent data and using clarity-driven allocation, the method improves estimation and adaptively prioritizes sensing. The current approach assumes known region partitions and sufficient data, which may be unrealistic under resource or time constraints. Future work includes data-driven region discovery, sample complexity analysis, and extension to decentralized settings with partial observability and limited communication.

\nocite{*}
\bibliographystyle{IEEEtran}
\bibliography{main.bib}

\appendix
\subsection{Extended derivation of \cref{lemma:variance_exact}}

We estimate the parameter $\rho$ of a variogram model $ \gamma(h;\rho) $ by fitting it to empirical semivariances at $N_b$ lag bins $h_1,\dots,h_{N_b}$. The empirical variogram is defined as:
\eqn{
\hat\gamma(h_b)
\;=\;
\frac{1}{2\,n_b}\sum_{(i,j)\in \mathcal{P}_b}\bigl(Y(x_i)-Y(x_j)\bigr)^2,
\quad
b=1,\dots,N_b,
}
where $\mathcal{P}_b$ denotes the set of co-temporal measurement pairs with spatial lag in bin $b$, and $n_b = |\mathcal{P}_b|$ is the number of such pairs.

\subsection*{1. Least squares estimator}

We fit a variogram model by minimizing the unweighted squared error:
\eqn{
L(\rho) =
\sum_{b=1}^{N_b}
\bigl[\hat\gamma(h_b)\,-\,\gamma(h_b;\rho)\bigr]^2.
}
Then the estimator is:
\eqn{
\hat\rho
=
\arg\min_{\rho}L(\rho).
}

\subsection*{2. First‐order optimality condition}

Differentiate $L(\rho)$ w.r.t.\ $\rho$ and set the result to zero at $\rho=\hat\rho$:
\begin{align}
\frac{dL}{d\rho}\Big|_{\rho=\hat\rho}
&=
\sum_{b=1}^{N_b} 2\,\bigl[\hat\gamma(h_b)-\gamma(h_b;\hat\rho)\bigr]
\;(-)\,\frac{\partial\gamma(h_b;\rho)}{\partial\rho}\Big|_{\rho=\hat\rho}
=0.
\end{align}

Define the residual vector $\mathbf{r}(\hat{\rho})\in\mathbb{R}^{N_b}$ and the Jacobian $\mathbf{J}(\hat{\rho})\in\mathbb{R}^{N_b\times1}$ by:
\eqn{
r_b(\hat{\rho}) = \hat\gamma(h_b) - \gamma(h_b;\hat{\rho}),
\qquad
J_b = \frac{\partial\gamma(h_b;\rho)}{\partial\rho}\Big|_{\rho=\hat{\rho}},
}
where $\rho_0$ is the true parameter. Then the first‐order condition becomes:
\eqn{
\mathbf{J}(\hat{\rho})^\top\,\mathbf{r}(\hat\rho)=0.
}

\subsection*{3. Taylor expansion of the residuals}

Expand $\mathbf{r}(\hat\rho)$ around $\rho_0$ via a first‐order Taylor approximation:
\eqn{
\mathbf{r}(\hat\rho)
=
\mathbf{r}(\rho_0)
+ 
\left.\frac{d\mathbf{r}}{d\rho}\right|_{\rho=\rho_0}
\,(\hat\rho-\rho_0)
+
\text{higher‐order terms}.
}
But
\eqn{
\frac{d\mathbf{r}}{d\rho}\Big|_{\rho_0}
= -\mathbf{J},
}
so neglecting higher-order terms,
\eqn{
\mathbf{r}(\hat\rho)
\approx
\mathbf{r}(\rho_0) \;-\;\mathbf{J}\,(\hat\rho-\rho_0).
}

\subsection*{4. Solve for \(\hat\rho-\rho_0\)}

Substitute into the first-order condition:
\eqn{
0
&=
\mathbf{J}^\top\!\left[\mathbf{r}(\rho_0)
- \mathbf{J}\,(\hat\rho-\rho_0)\right] \\
&=
\mathbf{J}^\top\mathbf{r}(\rho_0)
\;-\;
\mathbf{J}^\top\mathbf{J}\;(\hat\rho-\rho_0).
}
Rearrange:
\eqn{
\hat\rho-\rho_0
=
\bigl(\mathbf{J}^\top\mathbf{J}\bigr)^{-1}
\,
\mathbf{J}^\top\,\mathbf{r}(\rho_0).
}

\subsection*{5. Exact Variance}

Assume the empirical semivariances across bins are uncorrelated. Since each
\eqn{
\hat\gamma(h_b)
=\frac{1}{2\,n_b}\sum_{(i,j)\in \mathcal{P}_b}(Y_i - Y_j)^2
}
is an average over $n_b$ terms, we apply the variance‐of‐a‐mean rule:
\eqn{
\mathrm{Var}\bigl[\hat\gamma(h_b)\bigr]
= \frac{\sigma_b^2}{n_b}.
}
where $\sigma_b^2$ is the sample variance of squared differences in bin $b$. Alternatively, $\sigma_b^2$ is the variance of each squared term inside the bin.

Thus, the covariance of the residual vector $\mathbf{r}(\rho_0)$ is:
\eqn{
\mathrm{Cov}[\mathbf{r}(\rho_0)] = \Sigma,
\qquad
\Sigma = \mathrm{diag}\left(\tfrac{\sigma_1^2}{n_1},\dots,\tfrac{\sigma_{N_b}^2}{n_{N_b}}\right).
}

From earlier, we have the linear relationship:
\eqn{
\hat\rho-\rho_0
= A\,\mathbf{r}(\rho_0),
\quad
A = (\mathbf{J}^\top \mathbf{J})^{-1} \mathbf{J}^\top.
}
Applying the rule $\mathrm{Var}[A\,x] = A\,\mathrm{Cov}(x)\,A^\top$, we obtain:
\eqn{
\mathrm{Var}[\hat\rho]
= (\mathbf{J}^\top \mathbf{J})^{-1}
\,\mathbf{J}^\top\, \Sigma\, \mathbf{J}
\,(\mathbf{J}^\top \mathbf{J})^{-1}.
}

This completes the detailed derivation of the variance upper bound.

\subsection{Proof of \cref{lem:nls-rate}}

\begin{proof}
From \cref{lemma:variance_exact}, the variance of the least-squares estimator is given by:
\[
\mathrm{Var}[\lambda_{t,j}] = (J^\top J)^{-1} J^\top \Sigma J (J^\top J)^{-1},
\]
where $\Sigma \in \mathbb{R}^{N_b \times N_b}$ is a diagonal matrix whose $b$-th entry is the variance of the residuals in bin $b$ divided by the number of samples in that bin:
\[
\Sigma_{bb} = \frac{\sigma_b^2}{n_b}, \quad \text{with} \quad n_b \approx \frac{N_j(k)}{N_b}.
\]
By assumption, $\sigma_b^2 \leq \sigma_{\max}^2$ for all bins $b$, and $N_b$ is constant. Therefore, each diagonal entry of $\Sigma$ satisfies
\[
\Sigma_{bb} \leq \frac{\sigma_{\max}^2}{N_j(k)/N_b} = \frac{\sigma_{\max}^2 N_b}{N_j(k)}.
\]
Hence, the entire matrix $\Sigma$ scales as:
\[
\Sigma = \mathcal{O}(1/N_j(k)).
\]

Additionally, since $J$ is full rank and is assumed to have bounded entries, the product $J^\top J$ is invertible and its smallest eigenvalue is bounded away from zero, which implies:
\[
(J^\top J)^{-1} = \mathcal{O}(1), \quad \text{and} \quad J^\top \Sigma J = \mathcal{O}(1/N_j(k)).
\]
Putting these together:
\[
\mathrm{Var}[\lambda_{t,j}] = (J^\top J)^{-1} \left(J^\top \Sigma J\right)(J^\top J)^{-1} = \mathcal{O}(1/N_j(k)),
\]
and taking the square root yields:
\[
\mathrm{Std}[\lambda_{t,j}] = \mathcal{O}(N_j(k)^{-1/2}).
\]

Thus, the estimator converges to $\lambda_{t,j}^*$ in the mean-square sense. In particular, for any $\epsilon > 0$, we have
\[
\mathbb{P}\left( \left| \hat{\lambda}_{t,j}(k) - \lambda_{t,j}^* \right| > \epsilon \right) \to 0 \quad \text{as } N_j(k) \to \infty,
\]
i.e., $\hat{\lambda}_{t,j}(k) \to \lambda_{t,j}^*$ in probability. \qedhere
\end{proof}

\subsection{Some definitions}

\begin{definition}[Almost Sure Convergence]
Let \( X_1, X_2, \dots \) be a sequence of random variables defined on a probability space \( (\Omega, \mathcal{F}, \mathbb{P}) \). We say that \( X_n \) converges \textbf{almost surely} to a random variable \( X \), written
\[
X_n \xrightarrow{\text{a.s.}} X,
\]
if
\eqn{
\mathbb{P}\left( \lim_{n \to \infty} X_n(\omega) = X(\omega) \right) = 1.
}
\end{definition}

\begin{definition}[Strong Law of Large Numbers]
Let \( X_1, X_2, \dots \) be a sequence of independent and identically distributed (i.i.d.) random variables with finite mean \( \mu = \mathbb{E}[X_i] < \infty \). Then, the sample average converges almost surely to the expected value:
\[
\frac{1}{n} \sum_{i=1}^n X_i \xrightarrow{\text{a.s.}} \mu \quad \text{as } n \to \infty.
\]
That is,
\[
\mathbb{P} \left( \lim_{n \to \infty} \frac{1}{n} \sum_{i=1}^n X_i = \mu \right) = 1.
\]
\end{definition}

\end{document}